\documentclass[conference]{IEEEtran}
\IEEEoverridecommandlockouts

\usepackage{cite}
\usepackage{amsmath,amssymb,amsfonts}
\usepackage{graphicx}
\usepackage{textcomp}
\usepackage{xcolor}
\def\BibTeX{{\rm B\kern-.05em{\sc i\kern-.025em b}\kern-.08em
    T\kern-.1667em\lower.7ex\hbox{E}\kern-.125emX}}
\usepackage{times}
\usepackage{soul}
\usepackage{url}
\usepackage[hidelinks]{hyperref}
\usepackage[utf8]{inputenc}
\usepackage[small]{caption}
\usepackage{graphicx}
\usepackage{booktabs}
\usepackage{algorithm}
\usepackage{amsthm}

\usepackage[switch]{lineno}
\usepackage{hyperref}

\usepackage{url}
\usepackage{amsmath,amssymb}
\usepackage[square,sort,comma,numbers]{natbib}  
\usepackage{multirow}
\usepackage{graphicx}
\usepackage{caption}
\usepackage{url}
\usepackage{float}
\usepackage{array}
\usepackage{mathrsfs}
\usepackage{enumitem}
\usepackage{algorithm}
\usepackage{algorithmicx}
\usepackage{algpseudocode}
\usepackage{marvosym}
\usepackage{booktabs}
\usepackage{adjustbox}
\usepackage{epsfig}
\usepackage{tabularx}
\usepackage{colortbl}
\usepackage{pifont}
\usepackage{comment}
\usepackage{color}
\usepackage{verbatim}

\def\eqref#1{equation~\ref{#1}}

\def\1{\bm{1}}

\DeclareMathAlphabet{\mathsfit}{\encodingdefault}{\sfdefault}{m}{sl}
\SetMathAlphabet{\mathsfit}{bold}{\encodingdefault}{\sfdefault}{bx}{n}

\newtheorem{lemma}{Lemma}

\begin{document}

\title{CRTRE: Causal Rule Generation with Target Trial Emulation Framework}

\author{Junda Wang\textsuperscript{*}, Weijian Li\textsuperscript{†}, Han Wang\textsuperscript{‡}, Hanjia Lyu\textsuperscript{†}, Caroline P. Thirukumaran\textsuperscript{†}, \\ Addisu Mesfin\textsuperscript{†}, Hong Yu\textsuperscript{§}, and Jiebo Luo\textsuperscript{†} \\
\IEEEauthorblockA{\textit{\textsuperscript{*}University of Massachusetts Amherst, MA, USA} \\
\textit{\textsuperscript{†}University of Rochester, NY, USA}\\
\textit{\textsuperscript{‡}Shanghai Maritime University, Shanghai, China}\\
\textit{\textsuperscript{§}University of Massachusetts Lowell, MA, USA}\\
\textit{jundawang@umass.edu, \{weijianatusa, hanakookqwq\}@gmail.com,
hlyu5@ur.rochester.edu,}\\
\textit{\{caroline\_thirukumaran, addisu\_mesfin\}@urmc.rochester.edu, Hong.Yu@umassmed.edu, jluo@cs.rochester.edu}
}
}

\maketitle


\begin{abstract}
Causal inference and model interpretability are gaining increasing attention, particularly in the biomedical domain. Despite recent advance, decorrelating features in nonlinear environments with human-interpretable representations remains underexplored. In this study, we introduce a novel method called causal rule generation with target trial emulation framework ($CRTRE$),
which applies randomize trial design principles to estimate the causal effect of association rules. We then incorporate such association rules for the downstream applications such as prediction of disease onsets.
Extensive experiments on six healthcare datasets, including synthetic data, real-world disease collections, and MIMIC-III/IV, demonstrate the model's superior performance. Specifically, our method achieved a $\beta$ error of 0.907, outperforming DWR (1.024) and SVM (1.141). On real-world datasets, our model achieved accuracies of 0.789, 0.920, and 0.300 for Esophageal Cancer, Heart Disease, and Cauda Equina Syndrome prediction task, respectively, consistently surpassing baseline models. On the ICD code prediction tasks, it achieved AUC Macro scores of 92.8 on MIMIC-III and 96.7 on MIMIC-IV, outperforming the state-of-the-art models KEPT and MSMN. Expert evaluations further validate the model's effectiveness, causality, and interpretability.

\end{abstract}


\section{Introduction}
With the rapid growth of Machine Learning (ML) in the healthcare domain, 
methods have shown encouraging capability for solving a broad range of clinical tasks, such as disease understanding, diagnosis, and treatment planning, by leveraging a large number of Electric Health Record (EHR) data. Although these methods bring benefits to both patients and healthcare professionals~\citep{herpertz2017challenge, li2020predicting},  increased concerns on judgment errors~\citep{royce2019teaching,gandhi2006missed} as well as deficiency of understanding the workflow of ML systems~\citep{croskerry2013mindless} have become major road-blockers for the development and deployment of ML-based healthcare systems. An important factor underlines the aforementioned challenges is interpretability, i.e., that the black-box ML models are often associated with a limited capacity for performance analysis~\citep{ahmad2018interpretable}. Therefore, building interpretable ML models for healthcare becomes an imperative research direction. 

Numerous methods aimed at improving model interpretability have recently emerged, focusing on simplifying complex models to provide clearer insights into their decision-making processes~\citep{du2019techniques, zafar2019dlime, ribeiro2016model}. However, explanations of  black-box models often cannot be perfectly faithful to the original models and leave out much information which cannot be made sense of~\citep{rudin2019stop}. In addition, traditional ML models might be influenced by the data they are trained on, leading to unexpected bias and overfitting problems when applied to real-world environments.

Recently, methods propose associative inference~\citep{yu2024survey, zhang2021deep, kuang2020causal}, which has achieved promising results. However, these methods recognize diseases based on correlations and probability among patients' symptoms and medical history, while doctors prefer to diagnose 
according to 
the best causal explanations~\citep{imbens2015causal}. To help identify causal explanations, several methods have been proposed to address the agnostic distribution, including domain generalization which is becoming one of the most prominent learning paradigms~\citep{muandet2013domain}. Another study examines the distribution shift issue from a causal perspective, such as causal transfer learning~\citep{rojas2018invariant} and Structural Causal Model (SCM)~\citep{pearl2009causal} to identify causal variables based on the conditional independence test. Other researchers focus on more general methods under the stability guarantee by variable decorrelation through sample reweighting such as DWR~\citep{kuang2020stable, cui2022stable, kuang2020stable, zhang2021deep, kuang2018stable, kuang2021balance}. They leveraged co-variate balancing to eliminate the impact of confounding, assessing the effect of the target feature by reweighting the data so that the distribution of covariates is equalized across different target feature values. In spite of their advantageous analytical qualities, these approaches are usually limited to linear environments or binary datasets, and rarely employed in high-dimensional real-world applications due to the complex causal graph and strict assumptions. 

To address the aforementioned challenges, we propose a novel method, causal rule generation with target trial emulation framework ($CRTRE$),
which is interpretable and effective in both linear and nonlinear environments for stable prediction. CRTRE utilizes an association rule mining algorithm to extract rules as model features. To capture interactions among features, we employ a function \(F(x)\) and perform a Taylor expansion. A key aspect of our method is the use of a target trial emulation framework~\citep{hernan2022target}, 
to identify 
independence among variables in both linear and nonlinear context, which is essential for learning causal relationships accurately. Unlike previous methods that primarily eliminated linear relationships, CRTRE addresses nonlinear dependencies, leading to a more comprehensive understanding of variable interactions and ensuring stability in predictions. Our results show that CRTRE not only enhanced interpretability but also improved the performance of a broad range of clinical applications built upon both traditional ML and the recent AI models.

Experiments on synthetic datasets demonstrated that our method significantly improved model parameter estimation and enhanced prediction stability and effectiveness across varying distributions in nonlinear environments, thereby proving its superiority over existing approache ~\citep{kuang2020stable, kuang2020causal}
We also show that CRTRE improved prediction of disease onsets. Specifically, we evaluated CRTRE on three diverse real-world datasets  and demonstrated that our method substantially outperformed other state-of-the-art models such as SVM. 
To evaluate and compare CRTRE with the recent AI methods, we evaluated our method on the ICD code prediction task. 
CRTRE demonstrated superior performance, surpassing other state-of-the-art deep-learning-based models such as KEPT~\citep{yang2023multi, yang2022knowledge} and MSMN~\citep{yuan2022code} on both MIMIC-III~\citep{wang2020mimic} and MIMIC-IV datasets~\citep{johnson2023mimic}. 
Finally, to evaluate whether CRTRE outputs better association rules (for causality and interpretability), we asked three physicians with expertise in Cardiology, ENT, and Neurosurgery, respectively to rate the identified association rules. Our results show that physicians provided more favorable agreement with the association rules our method identified than existing methods~\citep{cui2022stable}.  

The main contributions of our work are as follows:

\begin{enumerate}
    \item[(1)] We expand the stable learning problem from linear environment to a nonlinear environment 
    so that stable learning can be widely applied in the real world. As far as we know, this is the first work in this direction.
    \item[(2)] This is the first study where we extend association rule generation with target trial emulation framework for causal rule identification. 
    \item[(3)] We demonstrate the superiority of CRTRE on both synthetic and real-world medical datasets for a broad range of clinical applications.
     \item[(4)] Our results show that physicians prefer our association rules for decision making, illustrating the clinical applicability of CRTRE. 
\end{enumerate}

\section{Related Work}
\subsection{Interpretability in Healthcare}

Increasing efforts have been devoted to Machine Learning (ML) interpretability, which is essential for healthcare applications. 
Among them, Generalized Additive Models (GAM)~\citep{hastie2017generalized} are a set of classic methods with univariate terms providing straightforward interpretabilities. GA$^{2}$M-model~\citep{lou2013accurate} brings additional capability for real-world datasets with the selected interacting pairs based on GAMs. 
\citep{lee2001review} identified association rules to assist physicians for clinical diagnoses and treatment plans of their patients. 

\citep{ahmed2021deep} applied association rules to identify important features from clinical images. \citep{sornalakshmi2021efficient} developed methods to incorporate association rules as features for downstream healthcare applications. However, most aforementioned methods generated association rules based on joint probability distribution, which has limited causal inference.

\subsection{Association Rule Mining}
Association rule mining identifies important association rules for downstream applications such as predictive modeling
~\citep{perccin2019arm,ordonez2006constraining}. 
We define symptoms $X$ as $X={x_{1}, x_{2},...,x_{n}}$ and disease outcome $Y$. 
$X \Rightarrow Y$ indicates that the disease $Y$ is related to the symptoms $X$.
We deployed three metrics to evaluate the significance of rules: $\emph{support}(X)=P(X)$ is the probability that the set appears in the total item set; $\emph{confidence}(X \Rightarrow  Y)=\emph{support}(X\cap  Y)/\emph{support}(X)$ is a measure of reliability; $\emph{lift}(X \Rightarrow Y)=\emph{confidence}(X \Rightarrow Y)/\emph{support}(Y)$ reflects the correlation between $X$ and $Y$ in the association rules~\citep{srikant1997mining}. In each association rule, $X$ is antecedent and $Y$ is the consequent. A rule that has a higher support and confidence has a stronger association rule. We defined a threshold of support and confidence. If the association rule passes the threshold, then we classify the rule as \textit{strong} association rule. 

Each association rule has two attributes: causality and association. An association may exhibit both high causality and association, or it may have a high association but low causality. Identifying the causal significance of an association rule, however, remains a critical challenge, particularly in contexts like diagnosis systems where causality is crucial. For example, traditional algorithms like Apriori~\citep{borgelt2002induction} and FP-Growth~\citep{han2000mining} focus on generating strong association rules based on correlations, often disregarding causality. These algorithms, though efficient in identifying frequent patterns or associations, do not address the need for causal understanding. Furthermore, \citep{yuan2017improved} proposed enhancements to the Apriori algorithm for greater efficiency, yet this method still focuses on correlation-based association rules. In medical diagnosis systems, this leads to inconsistencies between the rules generated by such algorithms and the causal reasoning doctors rely on. Therefore, developing methods to extract rules that reflect true causal relationships, rather than mere associations, remains an important challenge.

\begin{figure*}[t]
	\includegraphics[width=1\textwidth]{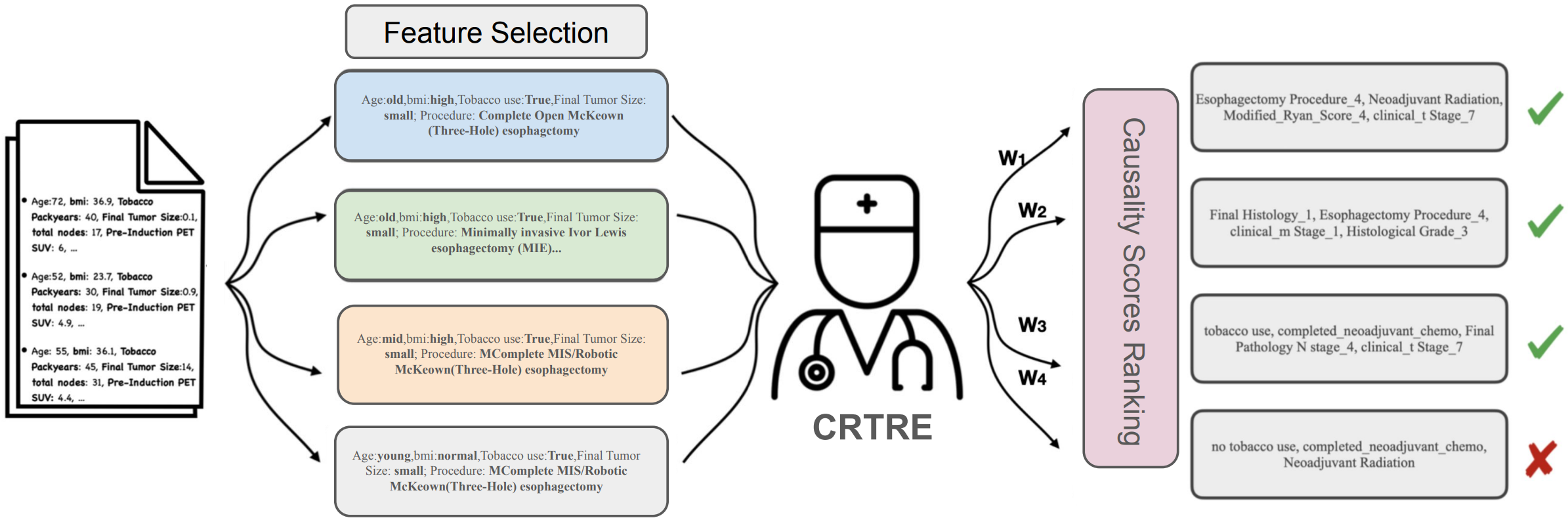}
	\centering
	\caption{We first extracted the features from clinical notes in tabular form, filtering out irrelevant attributes to focus on the most pertinent variables. Using the Apriori algorithm, we then generated association rules from the dataset, identifying significant associations between medical conditions, symptoms, diseases and treatments. After generating the initial set of rules, we pruned redundant or irrelevant ones to ensure relevance and quality. Finally, we applied our novel regularizer to score each rule, assessing its clinical relevance and statistical significance. This combination of the Apriori algorithm and our regularizer produced a concise, meaningful set of association rules, offering valuable insights for clinical decision-making.}
	\label{sample}
\end{figure*}

\subsection{Causal Inference}


One key challenge in healthcare is dealing with both observed and unobserved confounders across different environments~\citep{cui2022stable}. This has led to the growing popularity of causal inference methods, which are well-suited to addressing such issues. For instance, causal inference approaches that utilize network and hierarchical structures enable researchers to derive causal explanations from data~\citep{pearl2018theoretical}. A completely constructed causal graph among various features based on an unconfoundedness assumption that helps to reduce the influence of confounders~\citep{ma2021assessing}. 

For many machine learning algorithms, ensuring reliable performance depends on meeting two key assumptions. The first assumes that the test data is drawn from the same distribution as the training data, and the second requires the model to be correctly specified.  Nevertheless, in practical applications, our understanding of the test data and the underlying true model is typically limited. In situations where model specification is erroneous, the unknown distributional shift between training and test data results in inaccuracy in parameter estimation and instability in predictions on unknown test data. To address these problems, a Differentiated Variable Decorrelation (DVD) algorithm is proposed to eliminate the correlations of various variables in different environments by reweighting datasets~\citep{shen2020stable}. Moreover, \citep{xu2021stable} prove the effectiveness of stable learning and demonstrates the necessary of the stable prediction. 

A growing area of interest in machine learning is stable learning, which focuses on building models that perform consistently across different environments. Given various environments $\mathbf{e} \in E$ within datasets $D^{\mathbf{e}} = (X^{\mathbf{e}}, Y^{\mathbf{e}})$, the goal is to train a predictive model in one environment, $e_{i}$, that achieves uniformly 
minimal
error in a different environment, $e_{j}$. This is accomplished by learning the causal relationships between features $X^{e_{i}}$ and targets $Y^{e_{i}}$ under environment $e_{i}$. Researchers propose the Deep Global Balancing Regression (DGBR) algorithm~\citep{kuang2018stable} and  Decorrelated Weighting Regression (DWR) algorithm~\citep{kuang2020stable, cui2022stable} for stable prediction across unknown environments. They successively regard each variable as a treatment variable by using a balancing regularizer with theoretical guarantee.

In Equation~\ref{balancer}, $W$ is sample weight, $\mathbf{X}_{\cdot,j}$ is the $j^{th}$ variable in $\mathbf{X}$, and $\mathbf{X}_{\cdot,-j} = X/\{\mathbf{X}_{\cdot,j}\}$. With the global balancing regularizer in Equation~\ref{balancer}, a Global Balancing Regression algorithm is proposed to optimize global sample weights and causality for classification task.
\begin{equation}
\label{balancer}
\begin{array}{ll}
\min & \sum_{i=1}^{n} W_{i} \cdot \log \left(1+\exp \left(\left(1-2 Y_{i}\right) \cdot\left(\mathbf{X}_{i} \beta\right)\right)\right) \\
\text { s.t. } & \sum_{j=1}^{p}\left\|\frac{\mathbf{x}_{-j}^{T} \cdot(W \odot \mathbf{X}\cdot, j)}{W^{T} \cdot \mathbf{X}_{\cdot, j}}-\frac{\mathbf{X}_{i-j}^{T}(W \odot(1-\mathbf{X}\cdot, j))}{W^{T} \cdot\left(1-\mathbf{X}_{\cdot}, j\right)}\right\|_{2}^{2} \leq \lambda_{1}
\end{array}
\end{equation}

However, deploying the above algorithms on real-world datasets pose challenges. Methods like DGBR or DWR are designed to remove linear confounding by focusing on linear environments, which limits their effectiveness. In contrast, our approach not only handles linear relationships but also effectively eliminates nonlinear confounding, making it more suitable for complex real-world scenarios.

Target trial emulation has recently gained prominence as a robust methodology in observational studies, particularly for emulating randomized controlled trials (RCTs) where direct experimentation is unfeasible. ~\citep{hernan2016using} have highlighted the importance of this approach, emphasizing the structured design of observational analyses to mimic an RCT, thus improving causal inference. 
By explicitly specifying the trial protocol, such as eligibility criteria, treatment strategies, and follow-up procedures, researchers aim to reduce biases, including immortal time bias and confounding by indication. In this study, we incorporate the target trial emulation framework for identifying causal association rules. 

\section{Method}

We propose an interpretable model based on association rules and causal inference for EHR datasets to obtain the causality between features through a three-stage process.
\begin{enumerate}[label=(\roman*)]
    \item \textbf{Association and Transformation Rules Mining} We initially apply Apriori
    to generate the association rules.
    \item \textbf{Rule Selection} Since Apriori generates rules where redundancy is common, we develop rule-selection algorithm to select non-redundant rules.
    \item \textbf{Causality Computation} We introduce a novel $causal score$ to compute causality significance. Our method attempts to minimize the influence of confounding variables using the principles of emulation trial framework. Specifically, we followed inverse probability of censoring weighting by employing statistical techniques such as propensity score matching or adjustment methods to mitigate the confounding challenge.  
    By controlling for potential confounders, causal score reflects a more accurate estimation of the direct relationship, allowing rules with higher scores to be prioritized as potential causal rules.
\end{enumerate}
In the following we describe the details of each step.

\subsection{Association and Transformation Rules Mining} Initially, we implement the Apriori algorithm~\citep{agrawal1994fast} to identify association rules and construct a rule matrix for both positive and negative samples. This method counteracts the asymmetrical distribution of data. Rule representations, denoted as $<\mathscr{A}{i}, \mathscr{C}{i}|\theta_{i}>$, are subsequently constructed, where $\mathscr{A}{i}$ represents the antecedent of the rule $R_{i}$, $\mathscr{C}{i}$ signifies the consequent of the rule, and $\theta{i}$ indicates the confidence of the rule. We define a frequency function to compute the confidence of each rule: $\{R\} = \{\cup_{i}R_{i}\}=\{\cup_{i}frequent(\mathscr{A}_{i} \cup \mathscr{C}_{i})\}$.
 According to the rules generated, rule sets $\cup_{i}\{X_{i}:<\mathscr{A}_{i}, \mathscr{C}_{i}|\theta_{i}>\}$ are built for dataset $D$ where each rule is considered as a feature. Rule sets are then transformed into a zero-one matrix $X$ leveraging one-hot encoding.

\subsection{Rule Selection} Massive rules could be generated during the mining process, causing redundancy or even negative effects. To extract rules with strong correlations between features, we introduce an integer programming objective function:
\vspace{-2mm}
\begin{equation}
\begin{array}{lll}
\label{selection}
\text { Min }& \left\|W\right\|_{2}^{2}+\left\|\max(0, 1- Yh(x)\right\|_{2}^{2}\\
&h(x) = (W^{T}X\odot rep(\mathbb{I}(R > 0), n) \theta + b) \\
\text { s.t. } & \sum_{i} \mathbb{I}(R_{i} > 0) \leq \lambda_{1} & \\
& \sum_{i} \mathbb{I}(R_{i} > 0) \geq \lambda_{2}\\
& \{R_{i}\} \in \text{\{0, 1\}} &
\end{array}\\
\end{equation}

Here, $\odot$ denotes the Hadamard product, and $\mathbb{I}(R > 0)$ is an indicator function that converts $R$, a set of rules, into a binary vector of dimension $1*r$. The value of the indicator function is one when the rule is selected; otherwise, it is zero. The estimated parameters of the $rep(\mathbb{I}(R > 0),n)$, repeat function, are represented by $W$, and $b$ symbolizes the estimated bias. The function is designed to expand the vector $\mathbb{I}(R^{1*r} > 0)$ into a matrix of dimension $n*r$, where all rows mirror the first row. $\lambda_{1}$ and $\lambda_{2}$ represent the boundaries for the number of selected rules.

As we exclusively consider a binary classification problem in this context, which is a typical setup for most healthcare diagnostic issues, we use the inverse of the confidence of the negative class rule as the score. However, the number of rules identified by the association rule algorithm, such as Apriori, can be extensive, leading to an exceptionally high dimension of $R$ that cannot fit into Equation~\ref{selection}. Consequently, we remove one redundant rule at a time during each n-fold cross-validation run, based on a feature ranking criteria $w_{i}^{2}$.

Despite the systematic removal of redundant rules, redundant items within rules can still influence the model's performance. To address this issue, we initiate an iterative process to remove one item from each rule at a time, which results in an updated $R$ with a reduced dimension. We then rebuild the cross-validation sets and input data into SVM models to achieve an average accuracy. At every iteration, the item that enhances the model's average accuracy most significantly will be deleted. The code is shown in Appendix~\ref{app_algorithm} 

For comparison, we also built baseline models (i.e., SVM, Random Forest, Boosting algorithms, Logistic Regression, Neural Network(MLP), and DWR), using the traditional feature selection method Backward-SFS~\citep{ferri1994comparative}. Our results show that our method outperformed all aforementioned baseline models~\ref{resultaccuracy}.  



\subsection{Causality Computation} We follow the principle of the target trial emulation framework.  Specifically, we implement based on the assumption of DWR model, where the principal objective lies in the decorrelation of variables to discern causality. To adeptly manage the nonlinear relationships pervasive in real-world scenarios, we model these relationships utilizing Taylor expansion, represented by a function $\mathcal{F}(x)$ as depicted in Equation~\ref{taylor}. The derivatives of each fixed point can be interpreted as parameters to be resolved by transforming them into a polynomial fitting problem. This transformation is validated by the condition stipulating the equality of two polynomials solely when their degree and coefficients align.
\vspace{-2mm}
\begin{equation}
\begin{array}{ll}
\label{taylor}
x_{p_{1}} \sim & \mathcal{F}(x_{j})=f_{p_{1}p_{2}}(x_{p_{2}}(0))+f'_{p_{1}p_{2}}(x_{p_{2}}(0)) x_{p_{2}}+\\
& \frac{f''_{p_{1}p_{2}}(x_{p_{2}}(0))}{2!}x_{p_{2}}^{2}+\ldots+\frac{f^{(p)}_{p_{1}p_{2}}(x_{p_{2}}(0))}{p!}x_{p_{2}}^{p}+R_{p}(x_{p_{2}})
\end{array}
\end{equation}

where $x_{p_{1}}(0)$ and $x_{p_{2}}(0)$ are two different features which are expanded at 0 by using Taylor expansion. The elimination of the impact of intersecting areas is achieved by balancing the weight $W$ as is represented in Equation~\ref{polynomial}. If $x_{p_{1}}$ and $x_{p_{2}}$ are independent and nonlinearly uncorrelated, the derivatives of their relation functions are all 0: $\left\|\{\mathscr{F}_{p_{2}\rightarrow p_{1}}\} \text { / } \{f_{p_{1}p_{2}}(x_{p_{2}}(0))\} \right\|$ = 0 where $\mathscr{F}_{p_{2}\rightarrow p_{1}}$ are the relationship function between $x_{p_{1}}(0)$ and $x_{p_{2}}(0)$ can be calculated using Equation~\ref{solution}
\begin{small}
\begin{equation}
\begin{array}{ll}
\label{polynomial}
&\min_{\mathscr{F}_{p_{2}\rightarrow p_{1}}} R_{p}(x)^{2}\equiv \sum_{p_{1} \neq p_{2}} \sum_{i=1}^{n}\left[w_{i}x_{i p_{2}}-\mathcal{F}\left(w_{i}x_{i p_{1}}\right)\right]^{2} \\ \\
& \Rightarrow\mathscr{X}_{p_{2}}\left(w_{i}x_{p_{2}}\right) \mathscr{F}_{p_{2} \rightarrow p_{1}}=\mathscr{Y}_{p_{1}}
\end{array}
\end{equation}
\end{small}
\begin{scriptsize}
\begin{equation*}
    \mathscr{X}_{p_{2}} = \left[\begin{array}{cccc}
n & \sum_{i=1}^{n} w_{i}x_{ip_{2}} &\cdots& \sum_{i=1}^{n} w_{i}^{k}x_{ip_{2}}^{k} \\
\sum_{i=1}^{n} w_{i}x_{ip_{2}} & \sum_{i=1}^{n} w_{i}^{2}x_{ip_{2}}^{2} &\cdots& \sum_{i=1}^{n} w_{i}^{k+1}x_{ip_{2}}^{k+1} \\
\vdots & \vdots &\ddots& \vdots \\
\sum_{i=1}^{n} w_{i}^{k}x_{ip_{2}}^{k} & \sum_{i=1}^{n} w_{i}^{k+1}x_{ip_{2}}^{k+1} &\cdots& \sum_{i=1}^{n} w_{i}^{2 k}x_{ip_{2}}^{2 k}
\end{array}\right]
\end{equation*}
\end{scriptsize}

\begin{equation*}
\mathscr{F}_{p_{2}\rightarrow p_{1}}=\left[\begin{array}{c}
f_{p_{1}p_{2}}(x_{p_{2}}(0)) \\
f'_{p_{1}p_{2}}(x_{p_{2}}(0)) \\
\vdots \\
f^{(p)}_{p_{1}p_{2}}(x_{p_{2}}(0))
\end{array}\right],
\mathscr{Y}_{p_{1}} = \left[\begin{array}{c}
\sum_{i=1}^{n} y_{i} \\
\sum_{j=1}^{n} x_{i} y_{i} \\
\vdots \\
\sum_{i=1}^{n} x_{i}^{k} y_{i}
\end{array}\right]
\end{equation*}

\begin{equation}
\label{solution}
\mathscr{F}_{p_{2}\rightarrow p_{1}} = \left(\mathscr{X}_{p_{2}}^{\mathrm{T}} \mathscr{X}_{p_{2}}\right)^{-1} \mathscr{X}_{p_{2}}^{\mathrm{T}} \mathscr{Y}_{p_{1}}
\end{equation}

where $w_{i}$ is the weight for each sample and $n$ is the number of datasets. Combined with Figure~\ref{sample}, physical meaning can be given to the above variables: $r(w)$ is the regularizer term; $C$ is the factor to expand the influence of the intersection area to get the real causality comparing with $W$ applied to eliminate the influence of the public area: 
\begin{scriptsize}
\begin{equation}
\begin{array}{llll}
\text{Min}&  \sum_{i=1}^n (W_{i}+C)(\left(-y_i \log \left(\hat{p}\left(X_i\right)\right)-
 \left(1-y_i\right) \log \left(1-\hat{p}\left(X_i\right)\right)\right))\\
& \hat{p}\left(X_i\right)=\operatorname{expit}\left(X_i w+w_0\right)=\frac{1}{1+\exp \left(-X_i w-w_0\right)}, r(w) \leq \lambda_{3} \\
&\|\mathscr{F}_{p_{2}\rightarrow p_{1}, i>0}^{(i)}\|_{2}^{2} \leq \gamma,
\|W\|_{2}^{2} \leq \lambda_{1},\left(\sum_{k=1}^{n} W_{k}-1\right)^{2} \leq \lambda_{2}&
\end{array}
\end{equation}
\end{scriptsize}
When we have a smaller $\gamma$ value, the difference between $\beta$ and the true correlation coefficient (disjoint region and the target) will become smaller, resulting greater mutual information loss. 
\vspace{-5mm}
\begin{lemma}
If the number of features in the datasets and the terms in the Taylor expansion are fixed, when $n \to \infty$ there exists $W \succeq 0$(the proof is shown in appendix) such that
\begin{equation*}
     \lim_{n \to \infty}\|\mathscr{F}_{p_{2}\rightarrow p_{1}, i>0}^{(i)}\|_{2}^{2}
\end{equation*}
\end{lemma}

CRTRE is versatile and can be effectively integrated into various ML models and applications, including integrating Chain of Thought prompt with longformer~\citep{beltagy2020longformer} for the ICD code prediction task. The ICD prediction task involves predicting ICD codes from clinical notes. 

To achieve this, we first preprocess the clinical notes using MedSpaCy~\citep{medspacy}, a specialized NLP toolkit tailored for clinical text. MedSpaCy is employed to extract all Concept Unique Identifier (CUI) codes from the clinical notes. Each extracted CUI code serves as a feature, while the corresponding ICD code is treated as a label. This allows us to construct association rules between CUIs and ICD codes.

Next, we apply our regularization strategy and loss function to filter the association rules, specifically targeting those with strong causal relationships. The regularizer is designed to promote independence among variables in a nonlinear context, ensuring that only the most relevant and causally significant rules are retained. This filtering process is vital for enhancing the model's ability to learn meaningful and generalizable patterns from the data.

After identifying the causal rules, we match the CUIs extracted from the clinical notes to determine which rules are satisfied. These matched rules are then used to create prompts in a cloze-style format. We combine these prompts with ICD code descriptions and the original clinical notes to form a comprehensive input for the KEPTLongformer model, facilitating an effective reranking process. Specifically, we formulate the multi-label classification task as follows:
\vspace{-3mm}
\[
x_p = c_1, r_1 : [\text{MASK}], \, c_2, r_2 : [\text{MASK}], \, \ldots, \, c_n, r_n : [\text{MASK}]
\]

where \(r_i\) represents the satisfied rule set of ICD code candidates generated from previous models and \(c_i\) represents the description of each ICD. The [MASK] tokens serve as placeholders, which the model fills with specific tokens ("yes" or "no") to decide whether a code should be assigned to a clinical note based on the context provided.

To refine the predicted ICD code candidates generated by different models, we employ a cloze-style prompt approach rather than a generative prompt design typically used for few-shot learning. This reranking approach, inspired by code reranking techniques~\citep{yang2022knowledge}, uses the KEPTLongformer model to evaluate each candidate code's likelihood of being correct (positive or negative) for the given clinical note. By filling the [MASK] tokens with the appropriate vocabulary, KEPTLongformer determines if the code is appropriate for assignment.

\section{Experiment}

\subsection{Validation on Synthetic Dataset}\label{exp_datasets}

To examine the proposed constraints' effect on eliminating linear and nonlinear connotation relationships, we follow previous work~\citep{kuang2020stable} to conduct evaluations on synthetically generated datasets. Notice that a different objective function~\ref{regression} is built for regression task, where $W_{i}$ is the sample weight and the variable $\zeta$ is slack variable. In this experiment, we only expand two terms by the Taylor expansion.

\vspace{-5mm}
\begin{equation}
\label{regression}
\begin{array}{r}
\min _{w, b, \zeta, \zeta^*} \frac{1}{2} w^T w+ \sum_{i=1}^n (C+W_{i})\left(\zeta_i+\zeta_i^*\right) \\
\text { subject to } y_i-w^T \phi\left(x_i\right)-b \leq \varepsilon+\zeta_i \\
w^T \phi\left(x_i\right)+b-y_i \leq \varepsilon+\zeta_i^* \\
\zeta_i, \zeta_i^* \geq 0, i=1, \ldots, n\\
\|\mathscr{F}_{p_{2}\rightarrow p_{1}, i>0}^{(i)}\|_{2}^{2} \leq \gamma\\
\|W\|_{2}^{2} \leq \lambda_{1},\left(\sum_{k=1}^{n} W_{k}-1\right)^{2} \leq \lambda_{2}
\end{array}
\end{equation}

To test the stability of the algorithms, we generate a set of environment $e$ with a distinct distribution $P_{XY}$. Following Kuang's experiment~\citep{kuang2020stable}. In addition to the linear settings, we propose to include nonlinear evaluations under a nonlinear environment as shown in Appendix~\ref{settings} and the baselines are shown in Appendix~\ref{baseline_synthesis}

\noindent\textbf{Generating Various Environments} To test the stability of the algorithms, we generate a set of environment $e$ with a distinct distribution distribution $P_{XY}$. Following the Kuang's experiment~\cite{kuang2020stable}, we generate different environments based on various $P(S|V)$. To simplify the problem, we simulate $P(S_{b}|V$ on a subset $S_{b} \in S$, where the dimension of $S_{b}$ is $0.2*p$. We applied the bias rate equation $P r=\prod_{\mathbf{S}_i \in \mathbf{S}_b}|r|^{-5 * D_{i}}$ to tune the $P(S_{b}|V$, where $D_i=\left|f(\mathbf{S})-\operatorname{sign}(r) * \mathbf{V}_i\right|, r \in[-3,-1) \cup(1,3]$. $r > 1$ indicates that $Y$ and $S_{b}$ have positive unstable relationships, while $r < -1$ corresponds to the negative unstable relationships. The higher absolute value of $r$ the stronger connection between $S_{b}$ and $Y$, leading to generate different environments. The result is shown in Figure 3.

\begin{figure*}[t]
	\includegraphics[width=0.8\textwidth]{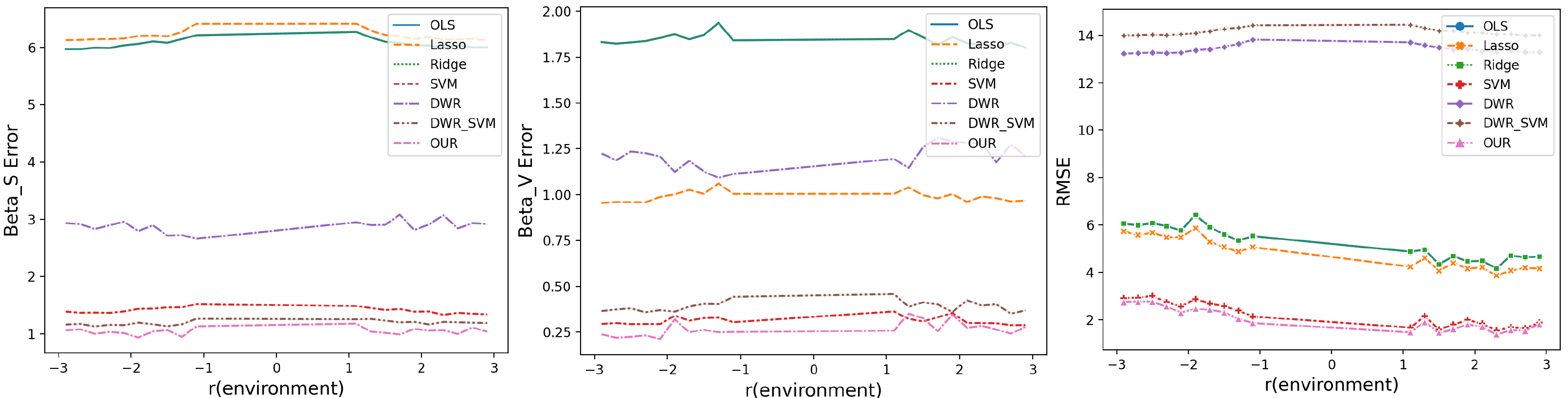}
	\centering
	\caption{Figures describe the $\beta_{S}$, $\beta_{V}$ and RMSE with various environments.}
	\label{shift}
 \label{causal_results_scores}
\end{figure*}

\subsection{Causal Experiment Results and Explanation}

\begin{figure*}[tb]
\begin{minipage}[t]{1\linewidth}
    \centering
    \includegraphics[scale=0.33]{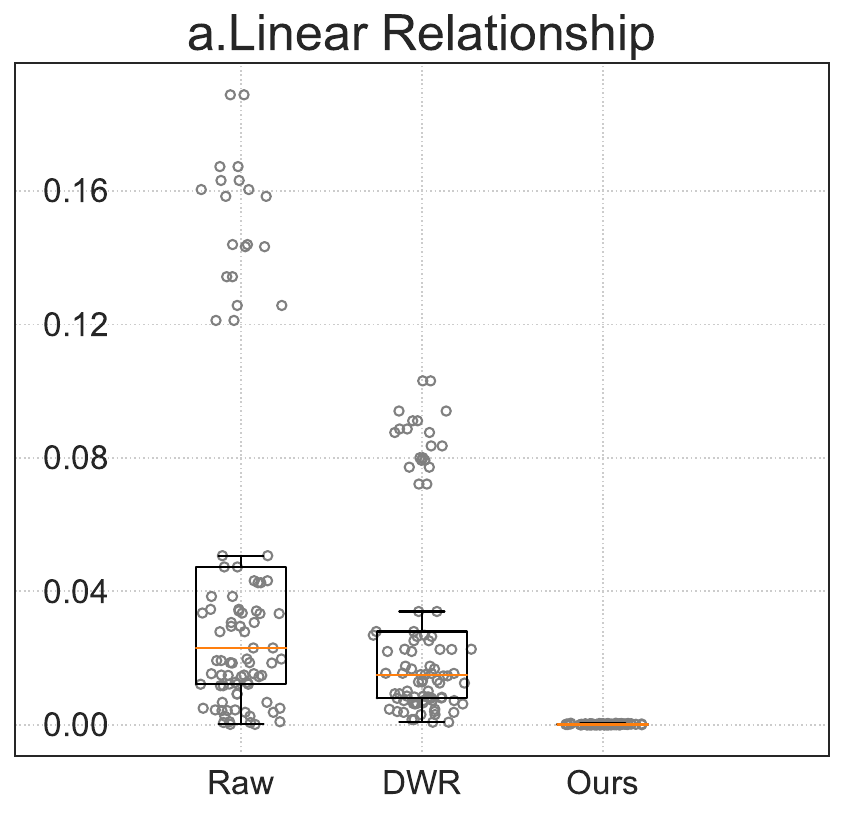}
    \centering
    \includegraphics[scale=0.33]{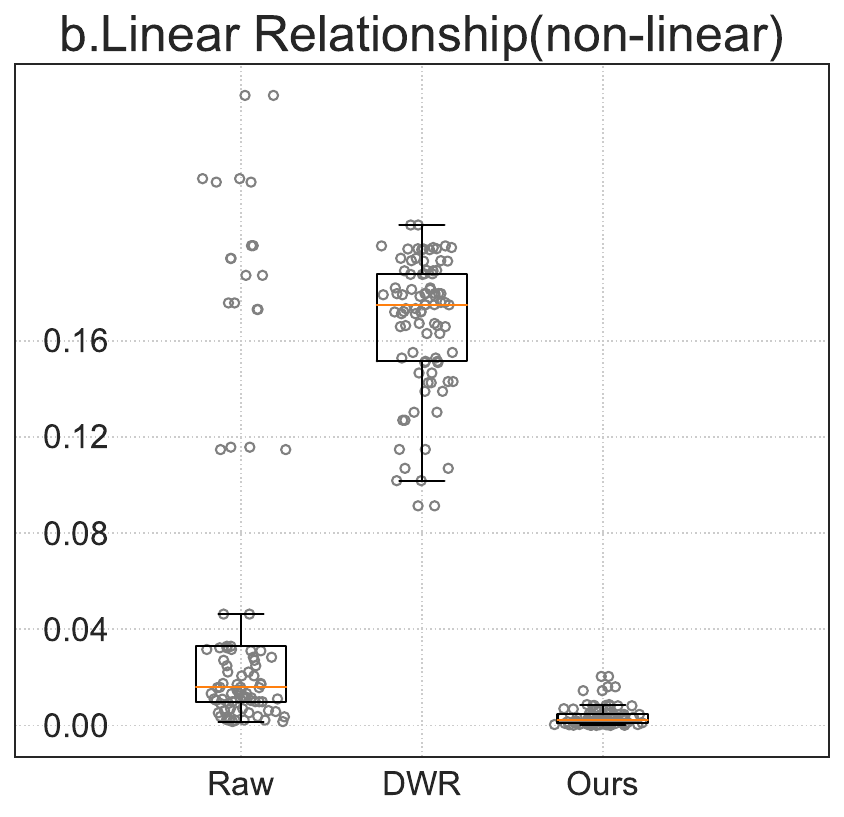}
    \centering
    \includegraphics[scale=0.33]{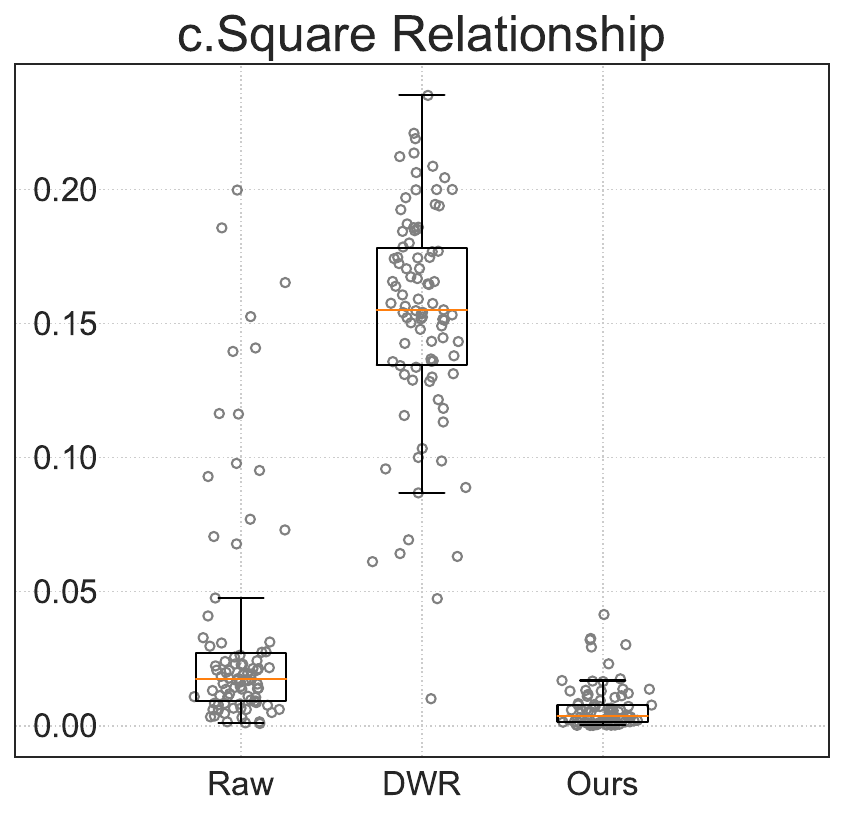}
    \centering
    \includegraphics[scale=0.33]{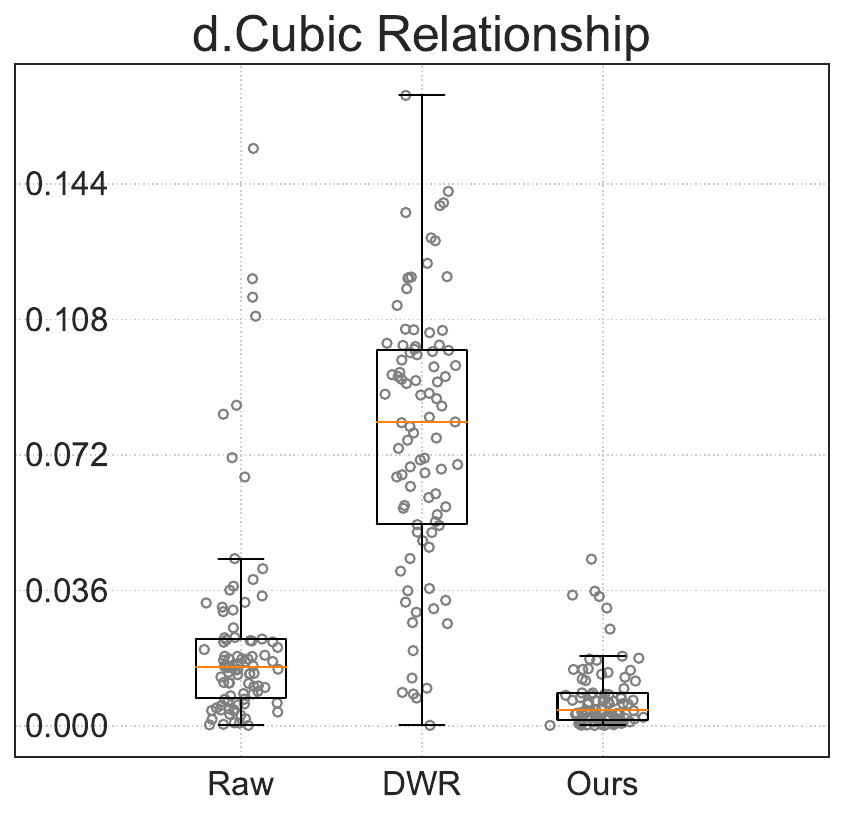}
    \centering
    \includegraphics[scale=0.33]{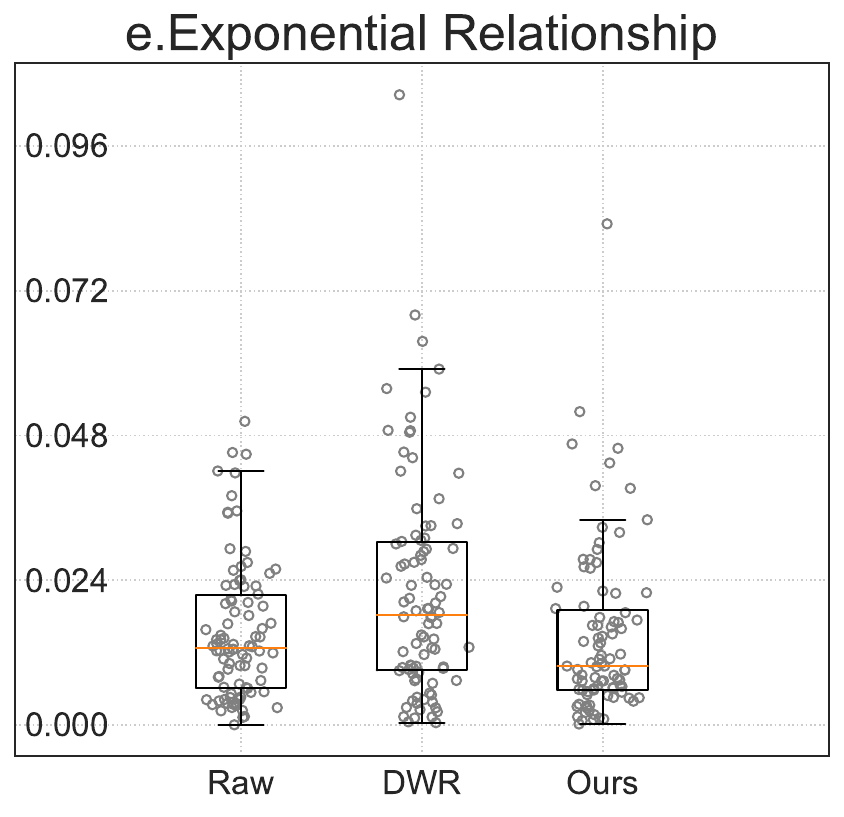}
    \centering
    \includegraphics[scale=0.33]{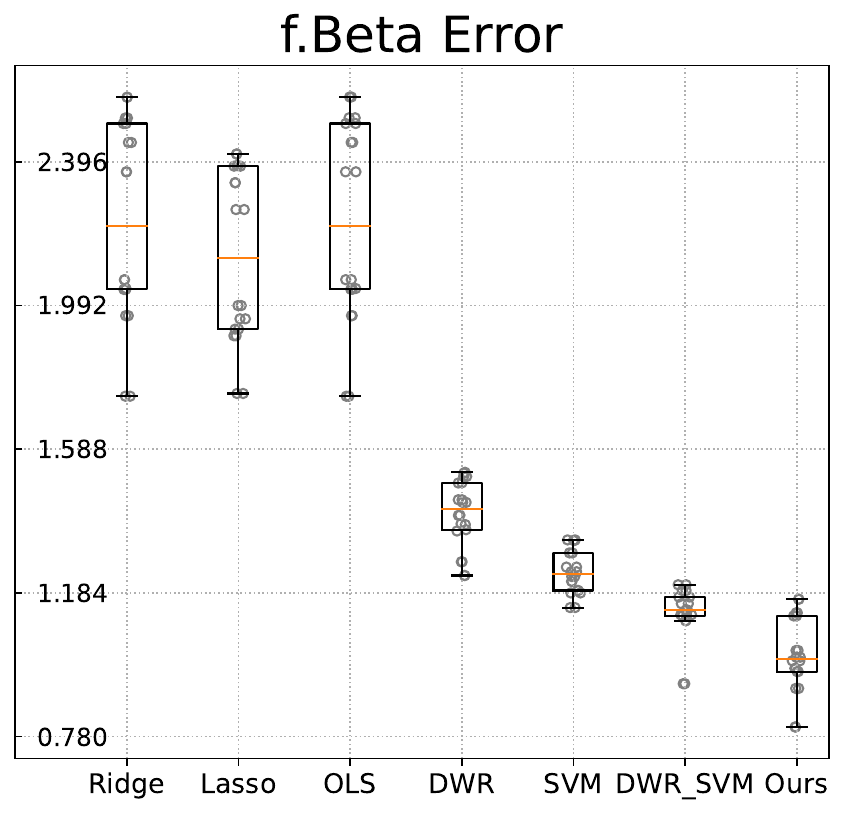}
  \end{minipage}%
  \caption{Figures (a)-(d) describe the distribution of the Pearson Coefficient values among various relationships. Figure (a) reports the $\beta$ errors of different models. Figure (f) is under a linear environment and other figures are under nonlinear environments. Our model is able to provide the greatest reduction of both linear and nonlinear relationships. }
  \vspace{-4mm}
  \label{beta_RMS}
\end{figure*}

To compare our regularizers with DWR, we apply Pearson Correlation to calculate the relationship strength among features. Since Pearson Correlation can only describe linear relationship, we construct nonlinear pairs $\mathbf{WV_{i}}$ with $\mathbf{(WV)^{2}_{j}}$, $\mathbf{(WV)^{3}_{j}}$ and $exp(\mathbf{WV}_{j})$ in addition to $\mathbf{WV_{i}}$ with $\mathbf{WV_{j}}$. The result can be found in Figure~\ref{beta_RMS}. Both DWR and the proposed regularizer can handle pure linear relationships (experimental environment(A)) but improvements are achieved from the proposed regularizer. As we add nonlinear relationships to the linear experimental environment, DWR start to have difficulty with the linear relationship part while the proposed method is still able to reduce a large amount of the relationships. For nonlinear environments, compared with the original unweighted dataset, DWR unexpectedly increases nonlinear relationships where there are no existing nonlinear relationships (square, cubic and exponential). Instead, our model does not incur nonlinear relationships but instead reduce them. To further validate its effectiveness on various sample sizes and the number of features, we calcualte the $\beta_{S}$ and $beta_{V}$ errors: $error(\beta)=\sum_{i}|\beta_{true}-\beta|$ , based on 9 kinds of datasets as shown in the figure~\ref{RMS_varying}. 

To further confirm that the coefficients estimated by our model are based on causality, we repeat experiments 50 times to calculate $\sum{\|\beta -\hat{\beta}\|}$, where $\beta$ and $\hat{\beta}$ represent the true value and estimated parameters, respectively. In Figure~\ref{beta_RMS}, we find that the difference between the estimated parameters and the true values is smaller with our model, compared to other models in the nonlinear environment.
Notice that our model achieves much smaller distribution variance as well as much smaller average values of $\beta$ errors comparing to baselines. Although the regularizer of DWR can solve the stable problem in linear environments, it retains or expands nonlinear confounding in the nonlinear environments. From the above results, we find that our model is able to reduce correlations among all predictors and avoid being affected by nonlinear confounding, resulting a reduced estimation bias in more general environments. 

\begin{table*}[t]
    \centering
    \caption{Results under varying sample size \textit{n} and number of variables within nonlinear environments.}
    \label{RMS_varying}
    \vspace{-1mm}
  \resizebox{0.7\linewidth}{!}{
	\begin{tabular}{l|lcc|lcc|lcc} 
		\toprule
			 Method & \multicolumn{3}{c|}{n=1000, m=5} & \multicolumn{3}{c|}{n=1000, m=10}   & \multicolumn{3}{c}{n=1000, m=15}   \\ 
   \cmidrule(r){2-10}
		& $\beta_{S}$ Error & $\beta_{V}$ Error       & $\beta$ Error & $\beta_{S}$ Error & $\beta_{V}$ Error & $\beta$ Error & $\beta_{S}$ Error       & $\beta_{V}$ Error & $\beta$ Error \\ 
  \midrule
        OLS & 3.357  & 0.430  & 1.894  & 3.605  & 0.729  & 2.167  & 3.823  & 0.866  & 2.345   \\ 
        Lasso & 3.390  & 0.326  & 1.858  & 3.586  & 0.647  & 2.117  & 3.940  & 0.390  & 2.165   \\ 
        Ridge & 3.357  & 0.430  & 1.893  & 3.604  & 0.729  & 2.166  & 3.822  & 0.866  & 2.344   \\ 
        SVM & 2.067  & 0.240  & 1.153  & 2.273  & 0.375  & 1.324  & 2.366  & 0.410  & 1.388   \\ 
        DWR & 2.279  & 0.249  & 1.264  & 2.566  & 0.658  & 1.612  & 3.258  & 1.182  & 2.220   \\ 
        DWR\_SVM & 1.799  & 0.303  & 1.051  & 2.077  & 0.483  & 1.280  & 2.494  & 0.918  & 1.706   \\ 
        OUR & \textbf{1.555}  & \textbf{0.199}  & \textbf{0.877}  & \textbf{1.898}  & \textbf{0.373}  & \textbf{1.135}  & \textbf{2.265}  & \textbf{0.382}  & \textbf{1.323}   \\ 
  \midrule
         & \multicolumn{3}{c|}{n=2000, m=5} & \multicolumn{3}{c|}{n=2000, m=10}   & \multicolumn{3}{c}{n=2000, m=15} \\
        \cmidrule(r){2-10}
		& $\beta_{S}$ Error & $\beta_{V}$ Error       & $\beta$ Error & $\beta_{S}$ Error & $\beta_{V}$ Error & $\beta$ Error & $\beta_{S}$ Error       & $\beta_{V}$ Error & $\beta$ Error  \\ 
    \midrule

        OLS & 3.253  & 0.444  & 1.849  & 3.521  & 0.630  & 2.075  & 4.071  & 0.561  & 2.316   \\ 
        Lasso & 3.278  & 0.250  & 1.764  & 3.490  & 0.473  & 1.982  & 4.260  & 0.168  & 2.214   \\ 
        Ridge & 3.253  & 0.444  & 1.848  & 3.520  & 0.630  & 2.075  & 4.071  & 0.561  & 2.316   \\ 
        DWR & 2.147  & 0.231  & 1.189  & 2.244  & 0.493  & 1.369  & 2.749  & 0.974  & 1.861   \\ 
        SVM & 2.020  & 0.271  & 1.145  & 2.158  & 0.315  & 1.237  & 2.453  & 0.349  & 1.401   \\ 
        DWR\_SVM & 1.675  & 0.305  & 0.990  & 1.861  & 0.407  & 1.134  & 2.317  & 0.572  & 1.445   \\ 
        OUR & \textbf{1.544}  & \textbf{0.214}  & \textbf{0.879}  & \textbf{1.719}  & \textbf{0.292}  & \textbf{1.006}  & \textbf{2.125}  & \textbf{0.323}  & \textbf{1.224}   \\ 

    \midrule
         & \multicolumn{3}{c|}{n=3000, m=5} & \multicolumn{3}{c|}{n=3000, m=10}   & \multicolumn{3}{c}{n=3000, m=15} \\
        \cmidrule(r){2-10}
		& $\beta_{S}$ Error & $\beta_{V}$ Error       & $\beta$ Error & $\beta_{S}$ Error & $\beta_{V}$ Error & $\beta$ Error & $\beta_{S}$ Error       & $\beta_{V}$ Error & $\beta$ Error  \\ 
    \midrule
        OLS & 3.297  & 0.335  & 1.816  & 3.593  & 0.579  & 2.086  & 3.736  & 0.611  & 2.173   \\ 
        Lasso & 3.279  & 0.074  & 1.677  & 3.803  & 0.179  & 1.991  & 3.703  & 0.527  & 2.115   \\ 
        Ridge & 3.297  & 0.335  & 1.816  & 3.593  & 0.579  & 2.086  & 3.735  & 0.611  & 2.173   \\
        DWR & 2.178  & 0.150  & 1.164  & 1.970  & 0.415  & 1.192  & 2.610  & 0.547  & 1.578   \\ 
        SVM & 2.066  & 0.217  & 1.141  & 2.046  & 0.338  & 1.192  & 2.261  & 0.329  & 1.295   \\ 
        DWR\_SVM & 1.764  & 0.284  & 1.024  & 1.833  & 0.312  & 1.072  & 2.082  & 0.484  & 1.283   \\ 
        OUR & \textbf{1.748}  & \textbf{0.065}  & \textbf{0.907}  & \textbf{1.618}  & \textbf{0.171}  & \textbf{0.894}  & \textbf{2.007}  & \textbf{0.325}  & \textbf{1.166}   \\ 
  \bottomrule
	\end{tabular}

 }
\vspace{-2mm}
\end{table*}

\subsection{Validation on Three Real-World Datasets}
To further validate the effectiveness of our model in real-world scenarios, we perform experiments on three different EHR datasets in appendix~\ref{dataset}. All data are carefully prepossessed to ensure no sensitive information is exposed. To further prove that our model can calculate causality rather than probability, we design a counterfactual experiment to test by calculating the accuracy, f1 score, precision and recall. We choose the ten least supported negative-sample-associated rules and select the the rule with the lowest absolute weight which is calculated by SVM classifier under normal training. We construct the spurious correlation between the rule and labels. For example, we set the chosen rule as $r_{i}$, and we pull all the samples in the training set that satisfy $\{r_{i}=1, y=1\}$ in the test set, and put all the samples in the test set that satisfy $\{r_{i}=0, y=0\}$ in the training set. This process will increase the value of the joint probability $P(r_{i}=1, y=1)$ and $P(r_{i}=0, y=0)$, which is contrary to the truth: $r_{i}=1 \Rightarrow y=0$. For each dataset, we tune the parameters so that all models perform as well as possible.

\begin{table*}[!t]
    \centering
    \caption{Prediction performances over various healthcare datasets on the counterfactual experiment.}
    \label{resultaccuracy}
  \resizebox{0.75\linewidth}{!}{
	\begin{tabular}{clcccc|lcccccc} 
		\toprule
			  &\multicolumn{5}{|c|}{Non Rule-based} & \multicolumn{6}{c}{Rule-based} \\ 
   \cmidrule(r){2-13}
      &\multicolumn{1}{|c}{XGBoost} & RF & SVM & LR & MLP & XGBoost & RF & SVM & LR & MLP & DWR & Ours  \\ 
     \cmidrule(r){1-13}
      &\multicolumn{11}{c}{Heart Disease} \\
      \cmidrule(r){1-13}
          \multicolumn{1}{c|}{Accuracy} & 0.903 & 0.887 & 0.885 & 0.869 & 0.947 & 0.869 & 0.868 & \textbf{0.960} & 0.934 & 0.878 & 0.937 & \textbf{0.960}  \\ 
          \multicolumn{1}{c|}{F1} & 0.880 & 0.863 & 0.899 & 0.882 & 0.952 & 0.882 & 0.879 & 0.963 & 0.940 & 0.892 & 0.943  & \textbf{0.964}  \\ 
          \multicolumn{1}{c|}{Precision} & 0.880 & 0.846 & 0.886 & 0.882 & 0.941 & 0.857 & 0.864 & \textbf{0.972} & 0.939 & 0.850 &0.931 & 0.966  \\ 
          \multicolumn{1}{c|}{Recall} & 0.880 & 0.880 & 0.912 & 0.882 & 0.963 & 0.909 & 0.897 & 0.956 & 0.945 & 0.940 & 0.958 & \textbf{0.964}  \\ 
          \multicolumn{1}{c|}{Causality} & - & - & - & - & - & 0.398 & 0.274 & 0.455 & 0.458 & 0.402 & 0.320 & \textbf{0.528}  \\ 
         \cmidrule(r){1-13}
      &\multicolumn{11}{c}{Esophageal Cancer} \\
      \cmidrule(r){1-13}
         \multicolumn{1}{c|}{Accuracy} & 0.788 & 0.750 & 0.827 & 0.808 & 0.750 & 0.738 & 0.727 & \textbf{0.900} & 0.812 & 0.846 & 0.854 & \textbf{0.900}  \\ 
          \multicolumn{1}{c|}{F1} & 0.776 & 0.683 & 0.809 & 0.800 & 0.735 & 0.708 & 0.697 & \textbf{0.888} & 0.783 & 0.824 &0.825& 0.885  \\ 
          \multicolumn{1}{c|}{Precision} & 0.704 & 0.737 & 0.827 & 0.808 & 0.720 & 0.723 & 0.692 & 0.867 & 0.804 & 0.843 & 0.842&\textbf{0.874}  \\ 
          \multicolumn{1}{c|}{Recall} & 0.864 & 0.636 & 0.792 & 0.833 & 0.750 & 0.699 & 0.713 & \textbf{0.913} & 0.771 & 0.812 &0.811 & 0.900  \\ 
          \multicolumn{1}{c|}{Causality} & - & - & - & - & - & 0.130 & 0.236 & 0.281 & 0.327 & 0.160 & 0.314 & \textbf{0.339}  \\ 
         \cmidrule(r){1-13}
      &\multicolumn{11}{c}{Cauda Equina Syndrome} \\
      \cmidrule(r){1-13}
         \multicolumn{1}{c|}{Accuracy} & 0.788 & 0.75 & 0.827 & 0.808 & 0.750 & 0.883 & 0.779 & 0.887 & 0.886 & 0.891 & 0.891 &\textbf{0.893}  \\
          \multicolumn{1}{c|}{F1} & 0.776 & 0.683 & 0.809 & 0.800 & 0.735 & 0.880 & 0.780 & 0.883 & 0.882 & \textbf{0.888} & 0.887 &\textbf{0.888}  \\ 
          \multicolumn{1}{c|}{Precision} & 0.704 & 0.737 & 0.827 & 0.808 & 0.720 & 0.818 & 0.706 & 0.825 & 0.822 & 0.831 & 0.831 & \textbf{0.834}  \\ 
          \multicolumn{1}{c|}{Recall} & 0.864 & 0.636 & 0.792 & 0.833 & 0.750 & 0.951 & 0.874 & 0.950 & 0.952 & \textbf{0.953} &\textbf{0.953} & 0.951  \\ 
         \multicolumn{1}{c|}{Causality} & - & - & - & - & - &0.231 & 0.298 & 0.279 & 0.132 & 0.262 & 0.308 & \textbf{0.477}\\

  \bottomrule
	\end{tabular}

}
\vspace{-3mm}
\end{table*}

\begin{table}[h]
\centering
\caption{Results on MIMIC-III-ICD9-50 datasets and MIMIC-IV-ICD9-50 datasets.}
\label{mimic}
\scalebox{0.80}{
\begin{tabular}{lcccccccccc}
\toprule
\multirow{2}{*}{Model} & \multicolumn{4}{c}{MIMIC-III-ICD9-50} & & \multicolumn{4}{c}{MIMIC-IV-ICD9-50} \\
 & \multicolumn{2}{c}{AUC} & \multicolumn{2}{c}{F1} & & \multicolumn{2}{c}{AUC} & \multicolumn{2}{c}{F1} \\
 \cmidrule(lr){2-5} \cmidrule(lr){7-10}
 & Macro & Micro & Macro & Micro & & Macro & Micro & Macro & Micro \\
\midrule
Joint LAAT & 92.36 & 94.24 & 66.95 & 70.84 & & 94.92 & 96.31 & 69.93 & 74.33 \\
MSMN & 	92.50 & 94.39 & 67.64 & 71.78 & & 95.13 & 96.46 & 71.85 & 75.78 \\
KEPT & 	92.63 & 94.76 & 68.91 & 72.85 & & 95.20 & 96.65 & 71.7 & 76.02 \\
CRTRE & \textbf{92.8} & \textbf{99.53} & \textbf{69.35} & \textbf{73.17} & & \textbf{95.39} & \textbf{96.70} & \textbf{72.21} & \textbf{76.34} \\
\bottomrule
\end{tabular}
}
\vspace{-5mm}
\end{table}

Based on diagnostic and procedure codes, patients with CES who underwent surgery between 2000 and 2015 were selected. Patient demographics (age, gender, race, comorbidities, and insurance status) and hospital characteristics (measured by hospital bed number quartiles).

\noindent\textbf{Pre-Processing:} We convert the continuous variables into categorical variables before feeding them to the model. To handle missing data in the datasets, we adopted MICE (Multiple imputations by chained equations) by transforming imputation problems into estimation problems where each variable will be regressed on the other variables. This method provides promising flexibility since every variable can be assigned a suitable distribution~\citep{wulff2017multiple}. Then we apply the SMOTE algorithm~\citep{fernandez2018smote} to address the class imbalance issue in our datasets.

\noindent\textbf{Feature Selection:} Redundant information in EHR datasets may cause noise and irrelevant information during feature extraction. A feature selection method~\citep{chen2007enhanced} is adopted. To improve the robustness of the model, we divide the dataset randomly into five groups for cross-validation. Each time we extract one group as the test set to analyze and measure the average performance in the feature selection process. Due to the high complexity of our model, we apply and compare the four baseline models: XGboost, SVM, Logistic Regression, and Random Forest to extract important features in feature selection and input the set of the features with the highest average AUROC scores into our model. 
In the end, we extract 13, 47, and 45 features for \emph{Heart Disease}, \emph{Esophageal Cancer}, and \emph{Cauda Equina Syndrome}, respectively.



\subsubsection{Results}
To measure the performance of models, we compute accuracy, precision, recall and F1 scores. The results are shown in Figure \ref{causal_results_scores}, and Table~\ref{resultaccuracy}. In addition to measuring the performance of the traditional models, we consider  the filtered rules as zero-one matrix $X$ into the baselines rather than the original datasets. In another counterfactual experiment as shown in Table~\ref{resultaccuracy}, we increase the spurious probability and construct the relationship which is contrary to the truth. We found that most of the baselines calculate parameters based on probability which results in worse performance, while our model is much higher than the other models. In addition, we extracted the weights of different rules learned from various models. We also asked doctors to evaluate these rules, as detailed in Appendix~\ref{humanevaluation}. We then calculated the alignment between the model-learned weights and the doctors' scores. To measure the similarity between them, we used Spearman Coefficients. Ultimately, we found that our model demonstrated a higher degree of similarity compared to other models, indicating that our model is better at learning the true causality across different environments.


\subsection{Validation on the ICD Coding Task}

In real-world experimental settings, we evaluated our model on both the MIMIC-III and MIMIC-IV datasets, achieving state-of-the-art results as shown in Table~\ref{mimic} in comparison with three other NLP methods for the ICD coding task (i.e., automatically assigning ICD codes based on the corresponding EHR note):
\begin{itemize}
    \item \textbf{Joint LAAT} \quad Joint LAAT~\citep{vu2020label} introduces a hierarchical joint learning approach, designed to predict both ICD codes and their parent codes. By leveraging the ICD code hierarchy, the model improves accuracy in medical code prediction.

    \item \textbf{MSMN} \quad MSMN utilizes synonyms with an adapted multi-head attention mechanism to achieve SOTA results on the MIMIC-III-50 task. This model captures richer semantic relationships, improving performance in medical code prediction.

    \item \textbf{KEPT} \quad KEPT tackles the ICD coding challenge with a prompt-based fine-tuning approach, addressing the long-tail distribution problem. By injecting domain-specific knowledge (hierarchy, synonym, and abbreviation), KEPT significantly improves performance on MIMIC-III datasets.
\end{itemize}

For the MIMIC-III-ICD9-50 dataset, our model achieved an AUC Macro of 92.8, outperforming Joint LAAT (92.36), MSMN (92.50), and KEPT (92.63). In terms of AUC Micro, our model achieved 99.53, compared to 94.24 for Joint LAAT, 94.39 for MSMN, and 94.76 for KEPT. Similarly, for F1 Macro, our model's achieved the best 69.35, outperforming Joint LAAT, MSMN, and KEPT ( 66.95, 67.64, and 68.91, respectively). On the MIMIC-IV-ICD9-50 dataset, our model demonstrated superior performance, achieving an AUC Macro of 95.39, compared to 94.92 for Joint LAAT, 95.13 for MSMN, and 94.97 for KEPT. Our model also achieved an AUC Micro of 96.70, surpassing Joint LAAT (96.31), MSMN (96.46), and KEPT (96.41). For F1 Macro, our model scored 72.21, outperforming Joint LAAT (69.93), MSMN (71.85), and KEPT (71.35). These results confirm the superior performance and generalizability of our model across different datasets and tasks, demonstrating the importance of causal rules.

\section{Conclusion}
In this paper, we present an interpretable causal inference approach focusing on nonlinear environments for healthcare applications. The proposed method extracts underlying association rules from the raw features as representations. A novel regularizer is constructed to handle both linear and nonlinear confoundings in real-world applications. The superior performances on four datasets (one synthetic and three real-world EHR datasets) from different domains compared to baseline methods validate both the effectiveness and generalizability of the proposed method. Consistent ratings between healthcare professionals and our method on the extracted rules on real-world datasets further validate the model's interpretability. Future works include exploring the applicability of adopting current framework into other Machine Learning applications. For example, we would like to extend the proposed causal rules extraction method as a general approach to provide nature language processing models the ability to analyze the causal relationship and enhance the model interpretability.

\bibliographystyle{IEEEtran}
\bibliography{reference.bib}

\newpage
\onecolumn
\appendix
\subsection{Algorithm}\label{app_algorithm}

We combine algorithm~\ref{rulereItemduce} with object function~\ref{selection} to select the robust rules and prune the redundant items. In the \emph{RulesSelection} function, we delete one rule each time with lowest $\|w\|^{2}_{2}$ and save the rule sets with the highest accuracy. In the \emph{ItemReduce} function, we apply cross-validation to train SVM model and save the item sets with best accuracy.
\renewcommand{\algorithmicrequire}{\textbf{Input:}}  
\renewcommand{\algorithmicensure}{\textbf{Output:}} 
\begin{algorithm}
        \caption{Rules Selection and Item Reduction}
        \label{rulereItemduce}
        \begin{algorithmic}[1] 
            
            \Require $Rules\{X_{i}\}$ are the association rules obtained by Apriori algorithm with training datasets. $data$ is EHR datasets.
            \Ensure $Bestrules$
            \Function{RulesSelection}{$Rules, data$}
                \State $Bestrules \gets Rules$
                \State $Objfunction$ is objective function
                \State $Select \gets Bestrules$
                \State $Bestaccuracy \gets Select$
                \State $Lastrules \gets \varnothing$
                \While{$Select \neq Lastrules$}
                    \State $Lastrules \gets Select$
                    \State $w \gets \operatorname*{argmin} Objectfunction(Select, data)$
                    \State $Selected \gets \operatorname*{argmin} w_{i}^{2}$
                    \State $Temprules \gets \{Bestrules\}/\{Selected\}$
                    \State $Tempaccuracy \gets Temprules$
                    \If{$Tempaccuracy > Bestaccuracy$}
                        \State $Bestaccuracy \gets Tempaccuracy$
                        \State $Select \gets Temprules$
                    \EndIf
                \EndWhile
                \State \Return{$Bestrules$}
            \EndFunction
            \Function{ItemReduce}{$Bestrules, data$}
                \State $Bestauc \gets SVM(Bestrules, data)$
                \State $Lastrules \gets \varnothing$
                \While{$Bestrules \neq Lastrules$}
                    \State $Item \gets \operatorname*{argmax} SVM(\{Bestrules\}/\{Item\})$
                    \State $Accuracy \gets SVM(\{Bestrules\}/\{Item\})$
                    \If{$Accuracy \geq Bestauc$}
                        \State $Bestauc \gets Accuracy$
                        \State $Bestrules \gets \{Bestrules\}/\{Item\}$
                    \EndIf
                \EndWhile
                \State \Return{$Bestrules$}
            \EndFunction
        \end{algorithmic}
\end{algorithm}

More details such as proof, algorithm or code could be see our anonymous github\footnote{More details are released at \url{https://anonymous.4open.science/r/Causal-Inference-via-Nonlinear-Variable-Decorrelation-for-Healthcare-Applications-20BF/supplementary.pdf}}

\subsection{Proof}
\section{Proof}
\label{proof}

\begin{lemma}
If the number of features in the datasets and the terms in the Taylor expansion are fixed, 
 when $n \to \infty$ there exists $W \succeq 0$ such that
\begin{equation*}
     \lim_{n \to \infty}\|\mathscr{F}_{p_{2}\rightarrow p_{1}, i>0}^{(i)}\|_{2}^{2}
\end{equation*}
\end{lemma}
\begin{proof}
Based on our regularizer, we know that
\begin{equation*}
\left(\begin{array}{cccc}
n & \sum_i w_i x_{i p_2} & \cdots & \sum_i w_i^k x_{i p_2}^k \\
\sum_i w_i x_{i p_2} & \sum_i w_i^2 x_{i p_2}^2 & \cdots & \sum_i w_i^{k+1} x_{i p_2}^{k+1} \\
\vdots & \vdots & & \sum_i \\
\sum_i w_i^k x_{i p_2}^k & \sum_i w_i^{k+1} x_{i p_2}^{k+1} & \ddots & \sum_i w_i^{2 k} x_{i p_2}^{2 k}
\end{array}\right)\left(\begin{array}{c}
f_{p_1 p_2}\left(x_{p_2}(0)\right) \\
f_{p_1 p_2}\left(x_{p_2}(0)\right) \\
\vdots \\
f_{p_1 p_2}^{(p)}\left(x_{p_2}(0)\right)
\end{array}\right)=\left(\begin{array}{c}
\sum_i y_i \\
\sum_i w_i x_{i p_2} y_i \\
\vdots \\
\sum_i w_i^k x_{i p_2}^k y_i
\end{array}\right)
\end{equation*}

We assume that the covariance is 0:
\begin{equation*}
\operatorname{cov}\left(\hat{x}_{i p_2}, y_i\right)=\operatorname{cov}\left(\hat{x}_{i p_2}^2, y_i\right)=\operatorname{cov}\left(\hat{x}_{i p_2}^3, y_i\right)=\cdots=\operatorname{cov}\left(\hat{x}_{i p_2}^k, y_i\right)=0
\end{equation*}
Combine with the following equation, we can get
\begin{equation*}
\begin{gathered}
n \rightarrow \infty: \frac{1}{n} \sum_n \hat{x}_{i p_2} y_i-\frac{1}{n^2} \sum_n \hat{x}_{i p_2} \sum_n y_i=\frac{1}{n} \sum_n \hat{x}_{i p_2}^2 y_i-\frac{1}{n^2} \sum_n \hat{x}_{i p_2}^2 \sum_n y_i \\
=\frac{1}{n} \sum_n \hat{x}_{i p_2}^k y_i-\frac{1}{n^2} \sum_n \hat{x}_{i p_2}^k \sum_n y_i=0
\end{gathered}
\end{equation*}

\begin{equation*}
\left(\begin{array}{cccc}
n & \sum_i \hat{x}_{i p_2} & \cdots & \sum_i \hat{x}_{i p_2}^k \\
\sum_i \hat{x}_{i p_2} & \sum_i \hat{x}_{i p_2}^2 & \cdots & \sum_i \hat{x}_{i p_2}^{k+1} \\
\vdots & \vdots & & \vdots \\
\sum_i \hat{x}_{i p_2}^k & \sum_i \hat{x}_{i p_2}^{k+1} & \ddots & \sum_i \hat{x}_{i p_2}^{2 k}
\end{array}\right)\left(\begin{array}{c}
f_{p_1 p_2}\left(x_{p_2}(0)\right) \\
f_{p_1 p_2}^{\prime}\left(x_{p_2}(0)\right) \\
\vdots \\
f^{(p)}{ }_{p_1 p_2}\left(x_{p_2}(0)\right)
\end{array}\right)=\frac{1}{n}\left(\begin{array}{c}
\sum_i y_i \hat{x}_{i p_2} \sum_i y_i \\
\vdots \\
\sum_i \hat{x}_{i p_2}^k \sum_i y_i
\end{array}\right)
\end{equation*}

\begin{equation*}
\left(\begin{array}{cccc}
\frac{\sum_i \hat{x}_{i p_2}^2}{\sum_i \hat{x}_{i p_2}}-\sum_i \hat{x}_{i p_2} & \frac{\sum_i \hat{x}_{p_2}^3}{\sum_i \hat{x}_{i p_2}}-\sum_i \hat{x}_{i p_2}^2 & \cdots & \frac{\sum_i \hat{x}_{i p_2}^{k+1}}{\sum_i \hat{x}_{i p_2}}-\sum_i \hat{x}_{i p_2}^k \\
\sum_i \hat{x}_{i p_2}^3 \\
\frac{\sum_i \hat{x}_{i p_2}^2}{2}-\sum_i \hat{x}_{i p_2} & \frac{\sum_i \hat{x}_{i p_2}^4}{\sum_i \hat{x}_{i p_2}^2}-\sum_i \hat{x}_{i p_2}^2 & \cdots & \frac{\sum_i \hat{x}_{i p_2}^{k+2}}{\sum_i \hat{x}_{i p_2}^2}-\sum_i \hat{x}_{i p_2}^k \\
\vdots & \vdots & \ddots & \vdots \\
\frac{\sum_i \hat{x}_{i p_2}^{k+1}}{\sum_i \hat{x}_{i p_2}^k}-\sum_i \hat{x}_{i p_2}^k & \frac{\sum_i \hat{x}_{i p_2}^{k+2}}{\sum_i \hat{x}_{i p_2}^k}-\sum_i \hat{x}_{i p_2}^k & \cdots & \frac{\sum_i \hat{x}_{i p_2}^{2 k}}{\sum_i \hat{x}_{i p_2}^k}-\sum_i \hat{x}_{i p_2}^k
\end{array}\right)\left(\begin{array}{c}
f_{p_1 p_2}\left(x_{p_2}(0)\right) \\
f^{\prime \prime} p_1 p_2\left(x_{p_2}(0)\right) \\
\vdots \\
f_{p_1 p_2}^{(p)}\left(x_{p_2}(0)\right.
\end{array}\right)=0
\end{equation*}

\begin{equation*}
\left|\begin{array}{cccc}
\frac{\sum_i \hat{x}_{i p_2}^2}{\sum_i \hat{x}_{i p_2}}-\sum_i \hat{x}_{i p_2} & \frac{\sum_i \hat{x}_{i p_2}^3}{\sum_i \hat{x}_{i p_2}}-\sum_i \hat{x}_{i p_2}^2 & \cdots & \frac{\sum_i \hat{x}_{i p_2}^{k+1}}{\sum_i \hat{x}_{i p_2}}-\sum_i \hat{x}_{i p_2}^k \\
\frac{\sum_i \hat{x}_{i p_2}^3}{\sum_i \hat{x}_{i p_2}^2}-\sum_i \hat{x}_{i p_2} & \frac{\sum_i \hat{x}_{i p_2}^4}{\sum_i \hat{x}_{i p_2}^2}-\sum_i \hat{x}_{i p_2}^2 & \cdots & \frac{\sum_i \hat{x}_{i p_2}^{k+2}}{\sum_i \hat{x}_{i p_2}^2}-\sum_i \hat{x}_{i p_2}^k \\
\vdots & \vdots & & \vdots \\
\frac{\sum_i \hat{x}_{i p_2}^{k+1}}{\sum_i \hat{x}_{i p_2}^k}-\sum_i \hat{x}_{i p_2}^k & \frac{\sum_i \hat{x}_{i p_2}^{k+2}}{\sum_i \hat{x}_{i p_2}^k}-\sum_i \hat{x}_{i p_2}^k & \cdots & \frac{\sum_i \hat{x}_{i p_2}^{2 k}}{\sum_i \hat{x}_{i p_2}^k}-\sum_i \hat{x}_{i p_2}^k
\end{array}\right| \neq 0
\end{equation*}

$\hat{x}_{i p_2}^2$ is influenced by the $w_{i}$ which can be adjusted, and the determinant of matrix is not equal to 0, hence the equation has only the trivial solution. We can get
\begin{equation*}
f_{p_1 p_2}^{\prime}\left(x_{p_2}(0)\right)=f_{p_1 p_2}^{\prime \prime}\left(x_{p_2}(0)\right)=\cdots=f_{p_1 p_2}^{(p)}\left(x_{p_2}(0)\right)=0
\end{equation*}

If we can prove under our regularizer, we can prove our method can work:

\begin{equation*}
n \rightarrow \infty:\left(\hat{x}_{i p_2}, y_i\right)=\operatorname{cov}\left(\hat{x}_{i p_2}^2, y_i\right)=\operatorname{cov}\left(\hat{x}_{i p_2}^3, y_i\right)=\cdots=\operatorname{cov}\left(\hat{x}_{i p_2}^k, y_i\right)=0
\end{equation*}

We set $\left(\hat{x}_{i p_2}, \hat{x}_{i p_2}^2, \ldots, \hat{x}_{i p_2}^k\right)$ is kernel density estimators: $g(x_{ip_{2}}$. We set the weight $w_{i}$ is:

\begin{equation*}
w_i=\frac{\prod_{i q} g\left(x_{i j}^q\right)}{\hat{G}\left(g\left(x_{i 1}\right), g\left(x_{i 2}\right), \ldots, g\left(x_{i p}\right)\right)}
\end{equation*}

\begin{equation*}
\begin{aligned}
n \rightarrow \infty: E\left[\hat{x}_{p_1}^q\right] &=\frac{1}{n} \sum_i x_{i p_1}^q \frac{\prod_{i q} g\left(x_{i j}^q\right)}{\hat{G}\left(g\left(x_{i 1}\right), g\left(x_{i 2}\right), \ldots, g\left(x_{i p}\right)\right)} \\
&=\int \ldots \int x_{i j}^q \prod_l g\left(x_{i l}^q\right) d x_{i 1} d x_{i 1}^1 \ldots d x_{i p}^q+o(1)=\int x_{i l}^{q_1} g\left(x_{i l}^{q_1}\right) d x_{i l}^{q_1}+o(1)
\end{aligned}
\end{equation*}

\begin{equation*}
\begin{aligned}
n \rightarrow \infty: E\left[\hat{x}_{p_1}^q, \hat{x}_{p_2}\right] &=\frac{1}{n} \sum_i x_{i p_1}^q x_{i p_1}\left(\frac{\prod_{i q} g\left(x_{i j}^q\right)}{\hat{G}\left(g\left(x_{i 1}\right), g\left(x_{i 2}\right), \ldots, g\left(x_{i p}\right)\right)}\right)^2 \\
=& \iint x_{i l}^{q_1} x_{i m} g\left(x_{i l}^{q_1}\right) g\left(x_{i m}\right) d x_{i l}^{q_1} d x_{i m}+o(1) \\
&=\int x_{i l}^{q_1} g\left(x_{i l}^{q_1}\right) d x_{i l}^{q_1} \int x_{i m} g\left(x_{i m}\right) d x_{i m}+o(1)
\end{aligned}
\end{equation*}

\begin{equation*}
n \rightarrow \infty: \operatorname{cov}\left(\hat{x}_{i p_1}^q, \hat{x}_{p_1}\right)=E\left[\hat{x}_{p_1}^q\right] E\left[\hat{x}_{p_1}\right]-E\left[\hat{x}_{p_1}^q, \hat{x}_{p_2}\right]=0
\end{equation*}

We can get:
\begin{equation*}
f_{p_1 p_2}^{\prime}\left(x_{p_2}(0)\right)=f_{p_1 p_2}^{\prime \prime}\left(x_{p_2}(0)\right)=\cdots=f_{p_1 p_2}^{(p)}\left(x_{p_2}(0)\right)=0
\end{equation*}

\end{proof}
\subsection{Baseline}
\label{baseline_synthesis}
We compare our model with five traditional methods for real-world prediction task:

\begin{itemize}
    \item \textbf{Logistic Regression} \quad We leverage the logistic regression classifier with L-BFGS solver for classification.
    \item \textbf{Random Forest} \quad We apply standard Random Forest classifier to solve the classification problem~\citep{pal2005random}.
    \item \textbf{XGboost} \quad We adopt XGBoost, an extreme gradient boosting methods, to compare with other models~\citep{chen2016xgboost}.
    \item \textbf{SVM} \quad We apply supervised learning models, SVM, with linear kernel to analyze data for classification~\citep{suykens1999least}.
    \item \textbf{MLP} \quad We use the traditional neural network multi-layer perceptron to solve this classification task~\citep{agatonovic2000basic}.
\end{itemize}

We conduct a series of ablative studies to evaluate the stability of our model. Table~\ref{hyperparameter} summarizes the results of the experiment with various values of $C$ and Lagrange penalty operators $\gamma$, and $\lambda$. For each sell, we fix the Lagrange penalty operators and increase $C$ to calculate the $\beta$ errors and RMSE errors. The higher $C$ indicates higher integrated mutual information is fed into the model and magnifies the impact of confounding. Higher $\gamma$ values will reduce more confounding effects and diminish mutual information. Here we choose the best parameters ($\gamma=600, \lambda=0.0005, C=0.5$) based on the smallest RMSE  also with a smaller $\beta$ error comparing to ($\gamma=1000, \lambda=0.0005, C=0.5$) which has the same RMSE.

\subsection{Datasets}
\label{dataset}
\noindent\textbf{Esophageal Cancer} consists of data from 261 patients who underwent esophagectomy for esophageal cancer between 2009 and 2018. The collected characteristics include patient demographics, medical and surgical history, clinical tumor staging, adjuvant chemoradiotherapy, esophagectomy procedure type, postoperative pathologic tumor staging, adjuvant chemoradiotherapy, postoperative complications, cancer recurrence, and mortality. 

\noindent\textbf{Cauda Equina Syndrome (CES)} is extracted from the Statewide Planning and Research Cooperative System (SPARCS)~\citep{new1984statewide}, a comprehensive database of all payers for all hospitalizations in New York State~\citep{joo202296}. 

\subsection{Experiment Settings}

For synthesis experiment, we compare our model with five baseline methods. For DWR-based methods, we adopt the code as well as the parameters published by the original authors:
\begin{itemize}
    \item Ordinary Least Square (OLS)~\citep{hutcheson2011ordinary}: 
    \begin{equation*}
    \min \|Y-\mathbf{X} \beta\|_2^2
    \end{equation*}

    \item Lasso~\citep{tibshirani1996regression}:
    \begin{equation*}
    \min \|Y-\mathrm{X} \beta\|_2^2+\lambda_1\|\beta\|_1
    \end{equation*}
    \item Ridge~\citep{hoerl1970ridge}:
    \begin{equation*}
    \min \|Y-\mathrm{X} \beta\|_2^2+\lambda_1\|\beta\|_2
    \end{equation*}
    \item Decorrelated Weighting Regression (DWR)~\citep{kuang2020stable}:
    \begin{small}
    \begin{equation*}
    \begin{array}{ll} 
        & \min _{W, \beta} \sum_{i=1}^n W_i \cdot\left(Y_i-\mathbf{X}_{i,} \beta\right)^2 \\
        \text { s.t } & \sum_{j=1}^p\left\|\mathbf{X}_{, j}^T \boldsymbol{\Sigma}_W \mathbf{X}_{,-j} / n-\mathbf{X}_{, j}^T W / n \cdot \mathbf{X}_{,-j}^T W / n\right\|_2^2<\lambda_2 \\
    \end{array}
    \end{equation*}
    \end{small}
    \item Support Vector Machines (SVM)~\citep{suykens1999least}:
        \begin{equation*}
            \min _{w, b, \zeta, \zeta^*} \frac{1}{2} w^T w+ \sum_{i=1}^n \left(\zeta_i+\zeta_i^*\right)
        \end{equation*}
    \item SVM combined with DWR(DWR\_SVM):
    \begin{small}
    \begin{equation*}
    \begin{array}{ll}
        &\min _{w, b, \zeta, \zeta^*} \frac{1}{2} w^T w+ \sum_{i=1}^n W_{i}\left(\zeta_i+\zeta_i^*\right) \\
        &\text { s.t } \sum_{j=1}^p\left\|\mathbf{X}_{, j}^T \boldsymbol{\Sigma}_W \mathbf{X}_{,-j} / n-\mathbf{X}_{, j}^T W / n \cdot \mathbf{X}_{,-j}^T W / n\right\|_2^2<\lambda_2 \\
    \end{array}
    \end{equation*}
    \end{small}
\end{itemize}

\label{settings}
\textbf{Linear Environment:} \quad For this setting, we construct features $\mathbf{S}$ that causes unstable $\mathbf{V}$ by auxiliary variables $z$ with linear relationship among features only: 
\begin{equation*}
\begin{array}{r}
\mathbf{Z}_{, 1}, \cdots, \mathbf{Z}_{, p} \stackrel{i i d}{\sim} \mathcal{N}(0,1), \mathbf{X}_{, 1}, \cdots, \mathbf{X}_{, p_{v}} \stackrel{i i d}{\sim} \mathcal{N}(0,1) \\
\mathbf{S}_{, i}=0.8 * \mathbf{Z}_{, i}+0.2 * \mathbf{Z}_{, i+1}, i=1,2, \cdots, p_{s}
\end{array}
\end{equation*}
\begin{equation*}
\mathbf{V}_{\cdot, j}=0.8 * \mathbf{X}_{\cdot, j}+0.2 * \mathbf{X}_{\cdot, j+1}+\mathcal{N}(0,1)
\end{equation*}

\textbf{Nonlinear Environment:} \quad In this setting, we combined square relationship and exponential relationship to generate various environment including potential nonlinear confounding to test our reweighted regularizer:
\begin{equation*}
\begin{aligned}
\mathbf{V}_{\cdot, j}&=\mathbf{X}_{\cdot, j} + 0.4 * \mathbf{X}_{\cdot, j+1} + 0.4 * exp(\mathbf{X}_{\cdot, j+1}) \\
&+ 0.4 * \mathbf{X}^{2}_{\cdot, j+1}+ 0.1 * \mathbf{X}^{3}_{\cdot, j+1} + \mathcal{N}(0,1) \\
\mathbf{S}_{\cdot, j} &=\mathbf{Z}_{\cdot, j}+0.4 * \mathbf{Z}_{\cdot, j+1} + 0.4 * exp(\mathbf{Z}_{\cdot, j+1})\\
&+ 0.4 * \mathbf{Z}^{2}_{\cdot, j+1} + 0.1 * \mathbf{Z}^{3}_{\cdot, j+1} + \mathcal{N}(0,1)
\end{aligned}
\end{equation*}

To further test the robustness of our algorithm, we assume that there are unobserved nonlinear terms, and construct the label $Y$ as shown in Equation~\ref{experiment}. Combined with weighed SVM loss function, we train our model to estimate the regression coefficient $\beta$. In this experiment, we set $\beta_{s}=\left\{\frac{1}{3},-\frac{2}{3}, 1,-\frac{1}{3}, \frac{2}{3},-1, \cdots\right\}, \beta_{v}=\overrightarrow{0}$, and $\varepsilon=$ $\mathcal{N}(0,0.3)$. In the experiment, we will set different dimension of $\beta$, hence if the dimension of $\beta_{S}$ is higher than 6, we will set the element of which index is larger than 6 as the $i\%6$-th of $\beta_{V}$.
\begin{equation}
\label{experiment}
Y_{p o l y}=f(\mathbf{S})+\varepsilon=[\mathbf{S}, \mathbf{V}] \cdot\left[\beta_{s}, \beta_{v}\right]^{T}+\mathbf{S}_{\cdot, 1} \mathbf{S}_{\cdot, 2} + \varepsilon
\end{equation}

\noindent\textbf{Heart Disease} is retrieved from the repository of the University of California, Irvine~\citep{asuncion2007uci}. We follow previous work to use 13 of 76 attributes: \emph{Age}, \emph{Sex}, \emph{cp}, \emph{threstbps}, \emph{chol}, \emph{fbs}, \emph{restecg}, \emph{thalach}, \emph{exang}, \emph{oldpeak}, \emph{slope}, \emph{cam} and \emph{thal}. 
\begin{table*}[t]
    \centering
    \caption{Hyperparameter study on the synthetic dataset. $\gamma$ and $\lambda$ are Lagrange penalty operators.}
    \label{hyperparameter}
  \resizebox{0.8\linewidth}{!}{
	\begin{tabular}{c|lcc|lcc|lcc} 
		\toprule
			 & \multicolumn{3}{c|}{$\gamma=600$, $\lambda=0.0001$} & \multicolumn{3}{c|}{$\gamma=600$, $\lambda=0.0005$}   & \multicolumn{3}{c}{$\gamma=600$, $\lambda=0.001$}   \\ 
   \cmidrule(r){2-10}
		& $C=0$ & $C=0.5$       & $C=1$ & $C=0$ & $C=0.5$       & $C=1$& $C=0$ & $C=0.5$       & $C=1$ \\ 
  \midrule
        $\beta_{S}$ Error & 1.956  & 1.919  & 1.996  & 1.769  & \textbf{1.926}  & 2.003  & 1.956  & 2.026  & 2.073   \\ 
        $\beta_{V}$ Error & 0.238  & 0.179  & 0.166  & 0.245  & 0.187  & 0.178  & 0.246  & 0.199  & 0.175   \\ 
        $RMSE$ Error & 4.943  & 4.732  & 4.680  & 4.854  & \textbf{4.726}  & 4.675  & 4.951  & 4.856  & 4.808   \\ 
  \midrule
         & \multicolumn{3}{c|}{$\gamma=800$, $\lambda=0.0001$} & \multicolumn{3}{c|}{$\gamma=800$, $\lambda=0.0005$}   & \multicolumn{3}{c}{$\gamma=800$, $\lambda=0.001$} \\
        \cmidrule(r){2-10}
		& $C=0$ & $C=0.5$   & $C=1$ & $C=0$ & $C=0.5$  & $C=1$& $C=0$ & $C=0.5$       & $C=1$ \\ 
    \midrule

        $\beta_{S}$ Error & 1.954  & 2.022  & 2.070  & 1.784  & 2.025  & 2.068  & 1.960  & 2.019  & 2.009   \\ 
        $\beta_{V}$ Error & 0.240  & 0.197  & 0.172  & 0.234  & 0.195  & 0.176  & 0.245  & 0.195  & 0.174   \\ 
        $RMSE$ Error & 4.945  & 4.859  & 4.825  & 4.849  & 4.860  & 4.793  & 4.961  & 4.858  & 4.674   \\ 

    \midrule
         & \multicolumn{3}{c|}{$\gamma=1000$, $\lambda=0.0001$} & \multicolumn{3}{c|}{$\gamma=1000$, $\lambda=0.0005$}   & \multicolumn{3}{c}{$\gamma=1000$, $\lambda=0.001$} \\
        \cmidrule(r){2-10}
		& $C=0$ & $C=0.5$       & $C=1$ & $C=0$ & $C=0.5$       & $C=1$& $C=0$ & $C=0.5$       & $C=1$ \\ 
    \midrule
        $\beta_{S}$ Error & 1.962  & 2.022  & 2.075  & 1.959  & 1.928  & 2.073  & 1.962  & 2.024  & 2.006   \\ 
        $\beta_{V}$ Error & 0.242  & 0.196  & 0.173  & 0.250  & 0.187  & 0.178  & 0.244  & 0.189  & 0.169   \\ 
        $RMSE$ Error & 4.938  & 4.859  & 4.812  & 4.950  & \textbf{4.726}  & 4.811  & 4.947  & 4.854  & 4.672   \\ 
  \bottomrule
	\end{tabular}

 }
\vspace{-1mm}
\end{table*}

\subsection{Human Evaluation}
\label{humanevaluation}
Quantitative performance does not always align with real-world practice. To ensure the quality, we ask doctors in Cardiology, ENT and Neurosurgery departments (three doctors in each) to rate each extracted rule based on their domain knowledge. Results can be found in Table~\ref{resultaccuracy}. Rule scores by the models are based on feature importance. Spearman Coefficients is adopted to compute rating consistency between our model and doctors. As can be observed, causality rankings of the baseline models vary greatly, indicating their unstable performances. However, our model is able to achieve consistent higher causal values, suggesting a better aligned rating mechanism with human experts and some examples can be found in the Table~\ref{toprules}. In the experiment, we sort the rules in descending order by calculating the importance and show the top five rules compared with the doctor's score in Table~\ref{toprules} (in our main paper). The scoring criteria are as follows:
\begin{itemize}
    \item \textbf{Score 4:} \quad Strongly agree that the rule contains causality.
    \item \textbf{Score 3:} \quad Agree that the rule contains causality.
    \item \textbf{Score 2:} \quad Disagree with this rule.
    \item \textbf{Score 1:} \quad Strongly disagree with this rule.
\end{itemize}

\begin{table*}[h!]
\centering
\caption{Rules filtered by algorithm are sorted in a descending order by our algorithm compared with the scores given by doctors.}
\label{toprules}
\resizebox{0.8\linewidth}{!}{
\begin{tabular}{llllllc}
\toprule
\multicolumn{6}{l}{Association Rules}  & Scores  \\ \midrule
\multicolumn{6}{l}{\textbf{Heart Disease}} & \\
\multicolumn{6}{l}{age middle, \#major vessels0, fixed defect, pressure normal, ST-T wave abnormality $\Rightarrow$ heart disease}  &  4  \\
 \multicolumn{6}{l}{age middle, cholesterol edge, \#major vessels0, lower than 120mg/ml $\Rightarrow$ heart disease}  & 3  \\
\multicolumn{6}{l}{non-anginal pain, cholesterol high, no exercise induced angina $\Rightarrow$ heart disease}  & 4  \\
\multicolumn{6}{l}{ST-T wave abnormality, downsloping $\Rightarrow$ heart disease}  & 4  \\
 \multicolumn{6}{l}{fixed defect, \#major vessels0, cholesterol edge $\Rightarrow$ heart disease}  & 4  \\
 \midrule
 \multicolumn{6}{l}{\textbf{Esophageal Cancer}} & \\
 \multicolumn{6}{l}{Modified Ryan Score 2.0, Esophagectomy Procedure 4 $\Rightarrow$ recurrence} & 2  \\ 
 \multicolumn{6}{l}{tobacco use, Alcohol Use, Neoadjuvant Radiation, Histological Grade 2, Final Histology 1  $\Rightarrow$ recurrence} & 4  \\ 
\multicolumn{6}{l}{Histological Grade 3, Neoadjuvant Radiation, Esophagectomy Procedure 4, Final Histology 1 $\Rightarrow$ recurrence}  & 4  \\ 
\multicolumn{6}{l}{clinical m Stage 1, Histological Grade 3, Neoadjuvant Radiation, Esophagectomy Procedure 4, Final Histology 1 $\Rightarrow$ recurrence}  & 4  \\ 
\multicolumn{6}{l}{esoph tumor location 4, Esophagectomy Procedure 5, Histological Grade 3 $\Rightarrow$ recurrence}   & 3  \\
 \midrule
\multicolumn{6}{l}{\textbf{Cauda Equina Syndrome}} & \\
\multicolumn{6}{l}{elixsum, beds, procedure 03 09 $\Rightarrow$ die360}  & 4  \\ 
\multicolumn{6}{l}{Emergency, diagnosis 344 60, complication 240days $\Rightarrow$ die360}  & 4  \\ 
\multicolumn{6}{l}{diagnosis 344 60, life threatening, complication 240days $\Rightarrow$ die360}  & 4  \\ 
\multicolumn{6}{l}{if aa $\Rightarrow$ die360} & 4  \\ 
\multicolumn{6}{l}{or potentially disabling conditions, complication 240days $\Rightarrow$ die360}  & 4  \\ 
\midrule
\end{tabular}
}
\end{table*}

\end{document}